\documentclass[11pt]{article}

\usepackage{amsfonts,amssymb,mathrsfs,amsmath,amssymb,amsthm,float}
\usepackage{graphicx}
\usepackage{geometry}
\usepackage{algorithm,algorithmic}
\usepackage{color}

\newtheorem{theorem}{Theorem}[section]
\newtheorem{lemma}{Lemma}[section]
\newtheorem{cor}{Corollary}[section]

\newtheorem{define}{Definition}[section]
\newtheorem{remark}{Remark}[section]

\newtheorem{exam}{Example}[section]
\newtheorem{assume}{Assumption}[section]

\def\R{{\mathbb{R}}}

\def\Z{{\mathbb{Z}}}

\def\E{{\mathbb{E}}}

\def\F{{\mathcal{F}}}
\def\H{{\mathcal{H}}}
\def\S{{\mathcal{S}}}
\def\Nn{{\mathcal{N}}}

\def\Relu{{\hbox{\rm{Relu}}}}
\def\Loss{{\hbox{\rm{Loss}}}}

\begin{document}


\title{\bf Improve the Robustness and Accuracy of Deep Neural Network with $L_{2,\infty}$ Normalization\thanks{This work is partially supported by NSFC grant No.11688101 and  NKRDP grant No.2018YFA0306702.}}
\author{Lijia Yu and Xiao-Shan Gao\\
 Academy of Mathematics and Systems Science, Chinese Academy of Sciences,\\ Beijing 100190, China\\
 University of  Chinese Academy of Sciences, Beijing 100049, China}
\date{\today}
\maketitle

\begin{abstract}
\noindent
In this paper, the robustness and accuracy of the deep neural network (DNN) was enhanced by introducing the $L_{2,\infty}$ normalization of the weight matrices of
the DNN with Relu as the activation function.
It is proved that the $L_{2,\infty}$ normalization leads
to large dihedral angles between two adjacent faces of the polyhedron graph of the DNN function and hence smoother DNN functions, which reduces over-fitting.
A measure is proposed for the robustness of a classification DNN, which is
the average radius of the maximal robust spheres with the sample data as centers.
A lower bound for the robustness measure is given in terms of the
$L_{2,\infty}$ norm. Finally, an upper bound for the Rademacher complexity of
DNN with $L_{2,\infty}$ normalization is given.
An algorithm is given to train a DNN with the $L_{2,\infty}$ normalization
and experimental results are used to show that the $L_{2,\infty}$ normalization
is effective to improve the robustness and accuracy.

\vskip 10pt
\noindent
{\bf Keywords}. Deep neural network, $L_{2,\infty}$ normalization, robust measure, over-fitting, smooth DNN, Rademacher complexity.
\end{abstract}



\section{Introduction}

The Deep neural network (DNN) \cite{lecun2015deep} has become the most  powerful method in machine learning, which was widely used in computer vision \cite{voulodimos2018deep}, natural language processing \cite{socher2012deep}, and many other fields.
A DNN is a composition of multi-layer neurons which are affine transformations together with nonlinear activation functions. The DNN  has universal power to approximate any continuous function over a bounded domain \cite{leshno1}.
Despite of its huge success, the DNN still has spaces for significant improvements
in terms of interpretability, robustness, over-fitting, and existence of adversary samples.
%


In this paper, we will focus on improving the robustness and reducing the over-fitting
of DNNs by introducing the $L_{2,\infty}$ normalization of the weigh matrices of the DNN.
We assume that the DNN uses Relu as the activation function, which is one
of the  widely used models of DNNs.
We give theoretical analysis of the $L_{2,\infty}$ normalization in three aspects:
(1) It is shown that the $L_{2,\infty}$ normalization leads to larger angles
between two adjacent faces of the polyhedron graph of the DNN function
and hence smooth DNN functions, which reduces the over-fitting.
(2) A measure of robustness for DNNs is defined and a lower bound
for the robustness measure is given in terms of the $L_{2,\infty}$ norm.
(3) An upper bound for the Rademacher complexity of the DNN
in terms of the $L_{2,\infty}$ norm is given.
Finally, an algorithm is given to train a DNN with $L_{2,\infty}$ normalization
and experimental results are used to show that the $L_{2,\infty}$ normalization
is effective to enhance the robustness and accuracy.
In the following, we give a brief introduction to the three aspects of  theoretical
analysis given in this paper and the related work.
%

When the complexity of the DNN surpasses the complexity of the target function, it may happen that the gap between the DNN and the actual function in the region outside the training data becomes larger along with the training, which is called over-fitting.
An over-fitting DNN fits the training data nicely but has low accuracy at the validation set.
A simple and effective way to reduce over-fitting is early termination~\cite[Sec.7.7]{DL}.
An important reason of over-fitting is that the DNN has too many parameters and an effective method to tackle this problem is the dropout method~\cite{dropout1,W2013,M2017}.
$L_1$ and $L_2$ regulations are also used to reduce the over-fitting~\cite[Sec.7.1]{DL}. The batch normalization is a recently proposed and important method to reduce over-fitting~\cite{I2015}.

In this paper, we propose a new approach to alleviate the over-fitting problem.
Intuitively, one of the reasons for over-fitting is that the DNN function has too much local fluctuations, and hence to make the DNN function smoother can reduce over-fitting.
When the activation function is Relu, the DNN could be considered as a real valued function and its graph is a polyhedron. We prove that the dihedral angle between two adjacent faces of the polyhedron graph of a DNN can be nicely bounded by the $L_{2,\infty}$ norm, and hence the $L_{2,\infty}$ normalization leads to smoother DNN functions and less over-fitting.
%
%
In~\cite{G2014}, the number of faces of the polyhedron graph is estimated
and is treated as a measure for the representation power of the DNN.
In this paper, we try to control the refined geometric structure of the graph polyhedron
to reduce the over-fitting  and to increase the robustness.

Robustness is a key desired feature for DNNs.
Roughly speaking, a network is robustness
if it has high accuracy at the input with little noise and
does not have a devastating hit when facing big noise.
It is clear that a more robust network is less possible to have
adversary examples~\cite{MA2017,H2015}.
%
%
Adding noises to the training data is an effective way to increase
the robustness~\cite[Sec.7.5]{DL}.
The $L_{1}$ regulation and $L_{1,\infty}$ normalization are used to increase the robustness of DNNs~\cite{Y2018}.
Knowledge distilling is also used to enhance  robustness and defend adversarial examples~\cite{H2015}.
In \cite{W2019}, linear programming is used to check whether a DNN is robust over a given region.
In \cite{C2016}, three methods were given to test the robustness of DNNs by experimental results.

In the above work, the robustness of the DNNs was usually demonstrated
by experimental results and there exists no measure for the robustness.
In this paper, we give a preliminary try on this direction by
defining a measure for the robustness of classification DNNs,
which is the average volume or radius of the maximal robust spheres with the training sample points as centers.
We give a lower bound for this measure of robustness
in terms of the $L_{2,\infty}$ norm, and show that
the lower bound is reasonably good for simple networks.

The Rademacher complexity is used to measure the richness of a class of real-valued functions.
When there exist no bias vectors, the Rademacher complexity of DNNs with $L_{p,q}$ normalization was estimated in \cite{N2014}.
In this paper, we give an upper bound for the
Rademacher complexity of general DNNs in terms of the $L_{2,\infty}$ norm.

The rest of this paper is organized as follows.
In Section \ref{sec-2}, the $L_{2, \infty}$ normalization of DNNs is defined.
In Section \ref{sec-3}, it is shown that the $L_{2, \infty}$ normalization
will lead to smoother graphs for the DNN.
In Section \ref{sec-4}, a measure of robustness is defined
and a lower bound for the measure is given.
In Section \ref{sec-5}, an upper bound for the Rademacher complexity for DNNs with  $L_{2, \infty}$ normalization is given.
In Section \ref{sec-6}, an algorithm and experimental results are given.

\section{$L_{2, \infty}$ normalization of DNN}
\label{sec-2}
In this section, we present the $L_{2, \infty}$ normalization of DNNs.

\subsection{The standard DNN}

A feed-forward deep neural network (DNN) can be represented by
\begin{equation}
\label{eq-dnn0}
\begin{array}{ll}
 x_{l}=f_{l}(W_{l}x_{l-1}+b_{l}) \in \R^{n_{l}}, l=1, 2, \ldots, L,\\
f_{l}:\R^{n_{l}}\rightarrow \R^{n_{l}},
 W_{l}\in \R^{n_{l}\times n_{l-1}}, b_{l}\in \R^{n_{l}}
\end{array}
\end{equation}
where $x_{0}\in\R^{n_0}$ is the input, $x_{L}\in\R^{n_L}$ is the output,
and $x_{l}=f^l(W_{l}x_{l-1}+b_{l}),  l=1, 2, \ldots, L-1$ is the $l$-th {\em hidden layer}.
$W_{l}$ is called the {\em weight matrix} and $b_{l}$ the {\em bias vector}. $f_{l}$ is a non-linear {\em activation function}. In this paper, we assume that $f^l$ is $\Relu$ when $l \in \{1,2,\dots,L-1\}$ and $f_{L}=I_{n_L}$ unless mentioned otherwise.
%

For simplicity, we assume that the input data are from $[0,1]^{n_0}$ or $\R^{n_0}$.
We denote this DNN by
\begin{equation}
\label{eq-dnn1}
\F:\R^{n_0}(\hbox{\rm{or}}\ [0,1]^{n_0} )\rightarrow\R^{n_L}\hbox{ with parameters }
\Theta_\F=\{W_{l}, b_{l}\,|\,l=1,\ldots,L\}.
\end{equation}
%
%
For instance,
a sample in MINST is a $28\cdot28 = 784$ dimensional vector in $[0,1]^{784}$.

Given a training data set $S=\{(x_{i},y_{i})\in \R^{n_{0}}\times\R^{n_{L}}\,|\, i=1,\ldots,K\}$,  we have the loss function   $\Loss=\frac{1}{K}\sum^{K}_{i=1}{L(\F(x_{i}), y_{i})}$, where
$L$ could be the cross entropy or other functions.
To train the DNN, we use gradient descent to update $\Theta_\F$ in order to make $\Loss$ as small as possible. In the end, we achieve the minimum point $\Theta_{0}$ such that $\bigtriangledown \F(\Theta_{0}) \approx0$, and obtain the trained DNN $\F$ with the parameters $\Theta_{0}$.

\subsection{The $L_{2, \infty}$ normalization}
The $L_{2, \infty}$ norm of a matrix $M$ is the maximum of the $L_{2}$ norm of the rows of M: $$||M||_{2, \infty}=\max_w \{||w||_{2}\,:\, w \,\hbox{\rm  is a row of }  M\}.$$
\begin{lemma}
\label{l11}
If the $L_{2, \infty}$ norm of two matrices $A,B\in\R^{n\times n}$ are both smaller than $c$, then the $L_{2, \infty}$ norm of   $A B$ is smaller than $\sqrt{n}*c^2$.
\end{lemma}
\begin{proof}
Let $a_{i,j}$ and $b_{i,j}$ be the entries of $A$ and $B$, respectively. The first row of $C=A  B$ is $(\sum^{n}_{i=1}a_{1,i}*b_{i,j})^{n}_{j=1}$.  Since $(\sum^{n}_{i=1}a_{1,i}*b_{i,j})^{2}\le (\sum^{n}_{i=1}(a_{1,i})^2)(\sum^{n}_{i=1}(b_{i,j})^2)$,
the square of the $L_{2}$ norm of the first row of $C$ is smaller than $$\sum^{n}_{j=1}(\sum^{n}_{i=1}a_{1,i}*b_{i,j})^{2}\le \sum^{n}_{j=1}(\sum^{n}_{i=1}(a_{1,i})^2)(\sum^{n}_{i=1}(b_{i,j})^2)=(\sum^{n}_{i=1}(a_{1,i})^2)*(\sum^{n}_{i,j=1}(b_{i,j})^2)\le c^2\times nc^2.$$
So the $L_{2, \infty}$ norm of  $A  B$ is smaller than $\sqrt{n}*c^2$.
\end{proof}

We now introduce a new constraint to the  DNN $\F$.
\begin{define}\label{l}
Let $\F$ be the DNN defined in \eqref{eq-dnn1} and $c\in\R_{>0}$.
The $||L||_{2, \infty}$ {\em normalization of $\F$} is defined to be
$||W_{l}||_{2, \infty}\leqslant c, l=1,\ldots,L$.
\end{define}

Note that the bias vectors are not considered in this constraint.

\section{Geometric meaning of $L_{2, \infty}$ normalization}
\label{sec-3}
In this section, we give the geometric meaning of the $L_{2, \infty}$ normalization
and show that it can be used to reduce the over-fitting of DNNs.

\subsection{Graph of DNN with Relu activation}
We assume that $\F$ is a standard DNN defined in \eqref{eq-dnn1} and $n_{L}=1$.
Since $n_L=1$, the DNN is a real valued function $\F: [0,1]^{n_0} \rightarrow \R$.
Since the activation function of $\F$ is Relu,  $\F$ is a piecewise linear function. Let $n_\F=n_0+1$. The graph of $\F$ is denoted by $\F_{g}$, which is a polyhedron in
$\R^{n_\F}$.

A {\em linear region} of $\F$ is a maximal connected open subset of the input space $[0,1]^{n_{0}}$, on which F is linear~\cite{G2014}.
Over a linear region, there exists a $W\in \R^{n_{0}}$ and $b\in \R$ such that $\F(x)=W^{\tau}x+b$. Thus, $F_{g}$ is a part of a hyperplane on every linear region, and we call this part of hyperplane  a {\em face of $\F_{g}$}.
%

\begin{lemma} The linear regions of $\F$ satisfy the following properties~\cite{G2014}.

(1) A linear region $R$ is an $n_{0}$-dimensional polyhedron  defined by a set of linear inequalities $C_i^\tau x+b_{i} > 0,i=1,\ldots,L_R$, where $C_i\in\R^{n_0}$ and $b_{i}\in \R$.


(2) The input space of $\F$  is the union of the closures of the linear regions of $\F$.
\end{lemma}
%

\begin{define}
\label{1255}
Two faces $A$ and $B$ of $\F_{g}$ are said to be {\em adjacent}
if $\overline{R}_A\cap \overline{R}_B$ is  an $(n_0-1)$-dimensional
polyhedron, where $\overline{R}_A$ and $\overline{R}_B$ are the closures
of $R_A$ and $R_B$, respectively.
\end{define}

When training the network $\F$, we randomly choose initial values and do gradient descent
to make the loss function as small as possible.
It is almost impossible to make $\Loss=0$ and we may assume that  the training terminates at a random point in the neighborhood of one of the minimal points of the loss function.
Therefore, the following assumption is valid for almost trained DNNs.

\begin{assume}
\label{ass-p}
The trained parameters of $\F$ are {\em random values} near   a minimum point of the loss function.
\end{assume}

\begin{lemma}[\cite{hodge1}]
\label{2}
Let $x_0$ be a randomly chosen point in $[0,1]^n$ and
$H(x)$ a polynomial in $n$ variables.
Then the  probability for $H(x_0)=0$ is zero.
\end{lemma}

Now we will give two lemmas under the Assumption \ref{ass-p}. Due to Lemma \ref{2}, the results  proved under Assumption \ref{ass-p}  are correct
with probability one. We will not mention this explicitly in the descriptions
of the lemmas and theorems below.
\begin{lemma}
\label{g1}
Under Assumption \ref{ass-p},  one row of the weight matrix $W^l$ is not the multiple of another row of $W^l$.
\end{lemma}
\begin{proof}
Let $W_{l,i,j}$ be the element  of $W_{l}$ at the $i$-th row and $j$-th column.
If the $i_{1}$-th row of $W_{l}$ is a multiple of the $i_{2}$-th row of $W_{l}$, then the parameters of these two rows must satisfy $W_{l,i_{1},j}\cdot W_{l,i_{1},k}=W_{l,i_{1},k}\cdot W_{l,i_{1},j}$ for all $j\ne k$.
By Lemma 3.2, probability for this to happen is zero.
\end{proof}

The following result shows that two adjacent faces are caused by a single $\Relu$ function.
\begin{lemma}
\label{g2}
Let $A,B$ be two  adjacent faces of $\F_{g}$
and $R_A,R_B$   their linear regions, respectively.
Under Assumption \ref{ass-p},  there exists a unique $\Relu$ function  $f_{c}$ of $\F$ such that
(1) all $\Relu$ functions of $\F$ except $f_{c}$ are either $0$
or positive over $R_A\cup R_B$;
(2) $f_{c}=0$ over $A$ and $f_{c}>0$ over $B$.
\end{lemma}
\begin{proof}
Assume the contrary: $\F$ has two $\Relu$ functions satisfying (1) and (2). We assume they are the $l_{1}$-th and $l_{2}$-th layers and $l_1<l_2$. Then we can calculate $F(x_{0})$ for   $x_{0}\in R_A\cup R_B$ by:
\begin{eqnarray*}
       x_{l_{1}}&=&W_{l_{1}}x_{0}+b_{l_{1}}\\
       x_{l_{1}+1}&=&W_{l_{1}+1}x_{l_{1}}+b_{l_{1}+1}+\Relu(h_{1}x_{l_{1}}+d_{1})\\
       x_{l_{2}}&=&W_{l_{2}}x_{l_{1}+1}+b_{l_{2}}\\
       x_{l_{2}+1}&=&W_{l_{2}+1}x_{l_{2}}+b_{l_{2}+1}+\Relu(h_{2}x_{l_{2}}+d_{2})\\
       F(x_{0})&=&W x_{l_{2}+1}+d
       \end{eqnarray*}
where  $W_{l_{j}}$,  $W_{l_{j}+1}$,  $W$ are weight matrices and  $b_{l_j}$,  $b_{l_{j}+1}$,  $d$ bias vectors.
In particular, $h_{j}$ is a matrix whose elements are $0$ except one row and $d_{i}$ is a vector whose elements are zero except one component, because the activity function changes sign when going from $R_{A}$ to $R_{B}$.

Assume the common boundary of  $R_A$ and $R_B$ has the normal vector $K$  and a point $x_{k}$ is on this boundary. Denote $f_c(M)$ to be the first nonzero row of the matrix $M$.
Then, we have the following relations.
\begin{eqnarray*}
&&h_{1}W_{l_{1}}x_{k}+h_{1}b_{l_1}+d_{1}=0\\
&&f_c(h_{1}W_{l_{1}})=K\\
&&h_{2}(W_{l_2}(W_{l_{1}+1}(W_{l_1}x_{k}+b_{l_{1}})+b_{l_1+1})+b_{l_2})+d_{2}=0\\
&&f_c(h_{2}W_{l_2}W_{l_{1}+1}W_{l_1})=K.
\end{eqnarray*}
Eliminating $x_{k}$ and $K$,  we have
\begin{eqnarray*}
&&f_c(h_{2}W_{l_2}W_{l_{1}+1}W_{l_1})=f_c(h_{1}W_{l_1})\\
&&f_c(h_{1}b_{l_1}+d_{1})=f_c(h_{2}W_{l_{2}}W_{l_{1}+1}b_{l_{1}}+h_{2}W_{l_2}b_{l_{1}+1}+h_{2}b_{l_2}+d_{2}).
\end{eqnarray*}
%
%
%
By Lemma \ref{2},  under Assumption \ref{ass-p}, $\F$ cannot satisfy these equations with probability $1$.
The lemma is proved.
\end{proof}

\subsection{$L_{2,\infty}$ normalization and dihedral angle between two adjacent faces of $\F_{g}$}
If the dihedral angle between two adjacent faces of $\F_{g}$ is large,
then $\F_{g}$ is a smooth polyhedron.
The dihedral angle between two hyperplanes $\sum^{n}_{i=1}a_{k, i}x_{i}+b_k=0, k=1, 2$
is
$$\arccos\frac{\sum^{n}_{i=1}a_{1, i}a_{2, i}}{\sqrt{\sum_{i=1}^{n}(a_{1, i})^{2}}\sqrt{\sum_{i=1}^{n}(a_{2, i})^{2}}}.$$
If the   hyperplanes are   directional,  then the dihedral angle will be  $\arccos\frac{\sum^{n}_{i=1}a_{1, i}a_{2, i}}{\sqrt{\sum_{i=1}^{n}(a_{1, i})^{2}}\sqrt{\sum_{i=1}^{n}(a_{2, i})^{2}}}$ or $\pi-\arccos\frac{\sum^{n}_{i=1}a_{1, i}a_{2, i}}{\sqrt{\sum_{i=1}^{n}(a_{1, i})^{2}}\sqrt{\sum_{i=1}^{n}(a_{2, i})^{2}}}$,
assuming $\arccos(x) \in [0,\pi]$.

We first consider a simple DNN with one hidden layer
$$\H:\R^{n}\rightarrow \R,\hbox{ where }\H(x)=L \cdot \Relu(Wx+b)+d.$$
Let $W=[w_{1}^{\tau}, w_{2}^{\tau}, \ldots,  w_{m}^{\tau}]^\tau\in \R^{m\times n}$,
$L^\tau=[l_{1}, l_{2}, \ldots,  l_{m}]^\tau\in \R^{m}$,
and $b^\tau=(b_{1}, b_{2}, \ldots,  b_{m})^\tau\in \R^{m}$. Then
\begin{equation}
\label{dnn-h1}
\H(x)=\sum_{i=1}^{m}l_{i}\ \Relu(w_{i}x+b_{i})+d.
\end{equation}

We can give an estimation of the dihedral angle between  two adjacent faces of $\H_{g}$.
\begin{lemma}
\label{d1}
If the $L_{2,\infty}$ norm of the weight matrices of $\H$ is smaller than $c$, then
under Assumption \ref{ass-p}, the dihedral angle between  two adjacent faces of $\H_{g}$ is bigger than $\pi-\arccos\frac{1}{\sqrt{1+c^{2}}}$ or $\pi-\arccos\frac{4-c^2}{4+c^2}$.
\end{lemma}
\begin{proof}
Let $x\in \R^{n}$ be the input and  $y\in \R$ the output.
From \eqref{dnn-h1}, the face of the polyhedron $\H_g$ has the expression -y+$\sum_{i=1}^{m}l_{i}  m_{i}  (w_{i}x+b_{i})+d=0$, where $m_{i}$ is $1$ or $0$ depending on whether $w_{i}x+b_{i}$ is bigger than $0$ or not.
For two adjacent faces $A_f$ and $B_f$, by Lemma \ref{g2},  there exists a unique $m_i$, which is respectively $0$ and $1$ over $A_R$ and $B_R$ (the linear regions of $A_f$ and $B_f$), and $m_{j}$ has fixed sign over these two linear regions for $j\neq i$.
%
So the two adjacent faces are defined by
$$y=W_{a}x+b_{a}\hbox{ and }y=W_{a}x+l_{i} w_{i}x+b_{a}+l_{i} b_{i}.$$
Since $W_{a}=\sum_{j=1}^m l_{j}m_{j}w_{j}$, the normal vectors of the two adjacent faces are $\{-1,W_{a}\}$ and $\{-1,W_{a}+l_{i} w_{i}\}$, respectively.
We denote
{\small $$D=\arccos\frac{1+||W_{a}||^2_{2}+l_iW_{a}\cdot w_i}{\sqrt{1+||W_{a}+l_{i}w_{i}||^{2}_{2}}\sqrt{1+||W_{a}||^2_{2}}}=\arccos\frac{1+||W_{a}||^2_{2}+l_iW_{a}\cdot w_i}{\sqrt{1+||W_{a}||_{2}^2+||l_{i}w_{i}||^{2}_{2}+2l_iW_a\cdot w_i}\sqrt{1+||W_{a}||^2_{2}}},$$}
where $\cdot$ is the inner product.
Since the first components of the normal vactors are $-1$, the dihedral angle between
the two faces are $\pi-D$. It suffices to prove $D\le\arccos\frac{1}{\sqrt{1+c^{2}}}$ or $D\le\arccos\frac{4-c^2}{4+c^2}$.

Define a univariate function: $G(x)=\frac{1+A+x}{\sqrt{1+A+B+2x}}$, where $A>0$, $B>0$ and $\sqrt{AB}\ge x\ge -\sqrt{AB}$. Then   $G'(x)=\frac{B+x}{(1+A+B+2x)^{1.5}}$. That is, the minimum point of G(x) is $-\sqrt{AB}$ or $-B$.
%
We will use this function to analyse the dihedral angle. We divide the discussion into two cases.

Case 1: assume $||W_a||^2_2\ge||l_iw_i||^2_2$.
Let $||W_a||^2_2=A$, $||l_iw_i||^2_2=B$ and use the property for function $G(x)$.
When $||l_iw_i||^2_2=-2l_iW_a\times w_i$, the angle achieves the maximal value:
$$D\le\arccos\frac{1+||W_{a}||^2_{2}-||l_iw_i||^2_2} {\sqrt{1+||W_{a}||_{2}^2-||l_{i}w_{i}||^{2}_{2}}\sqrt{1+||W_{a}||^2_{2}}}=\arccos\frac{\sqrt{1+||W_{a}||^2_{2}-||l_iw_i||^2_2}}{\sqrt{1+||W_{a}||^2_{2}}}$$
We have $||l_iw_i||^2_2\le c^2$ due to the $L_{2,\infty}$ normalization and
the assumption $||W_a||^2_2\ge||l_iw_i||^2_2$.
Then we have
$$\arccos\frac{\sqrt{1+||W_{a}||^2_{2}-||l_iw_i||^2_2}}{\sqrt{1+||W_{a}||^2_{2}}}\le\arccos\frac{\sqrt{1+||l_iw_i||^2_{2}-||l_iw_i||^2}}{\sqrt{1+||l_iw_i||^2_{2}}}\le\arccos\frac{1}{\sqrt{1+c^2}}.$$
This proves the first bound in the lemma.

Case 2: assume $||W_a||^2_2<||l_iw_i||^2_2$.
By the property of function $G(x)$, when $l_iW_a\cdot w_i=-||W_{a}||_{2}||l_iw_i||_2$, the angle is the biggest:
$$D\le \arccos\frac{1+||W_{a}||^2_{2}-||W_{a}||_{2}||l_iw_i||_2} {\sqrt{1+||W_{a}||_{2}^2+||l_{i}w_{i}||^{2}_{2}-2||W_{a}||_{2}||l_iw_i||_2}\sqrt{1+||W_{a}||^2_{2}}}.$$
Denote $K=||l_iw_i||_2-||W_{a}||_{2}$ and $T=||W_{a}||_{2}$. Then it becomes
$$\arccos\frac{1-TK} {\sqrt{1+K^2}\sqrt{1+T^2}}.$$
Since $0\le T \le T+K \le c$,  we have
$$\arccos\frac{1-TK} {\sqrt{1+K^2}\sqrt{1+T^2}}\le\arccos\frac{4-c^2}{4+c^2}.$$
This is because $\frac{1-TK} {\sqrt{1+K^2}\sqrt{1+T^2}}\ge \frac{4-c^2}{4+c^2}$,
which will be proved below.
Taking square and simplifying, it becomes
$$(\frac{4+c^2}{4-c^2})^2(TK-1)^2\ge 1+T^2+K^2+T^2K^2.$$
We know $c^2\ge T^2+K^2+2TK$ and $(TK-1)^2\ge(1-c^2/4)^2$. Then we have
$$1+T^2+K^2+T^2K^2\le c^2+(TK-1)^2 \le (\frac{4+c^2}{4-c^2})^2(TK-1)^2.$$
We prove the second bound in the lemma.
%
\end{proof}

\begin{remark}
When $c$ is small, $\arccos\frac{1}{\sqrt{1+c^{2}}}$ and $\arccos\frac{4-c^2}{4+c^2}$ is approximately $0$ and hence the dihedral between two adjacent faces of $\H_{g}$ will be large,
as shown by Figures \ref{fig-2} and \ref{fig-3}, which show that the $L_{2,\infty}$ normalization makes the DNN function smooth.
%
\end{remark}

\begin{figure}[ht]
\begin{minipage}[t]{0.49\linewidth}
\centering
\includegraphics[scale=0.20]{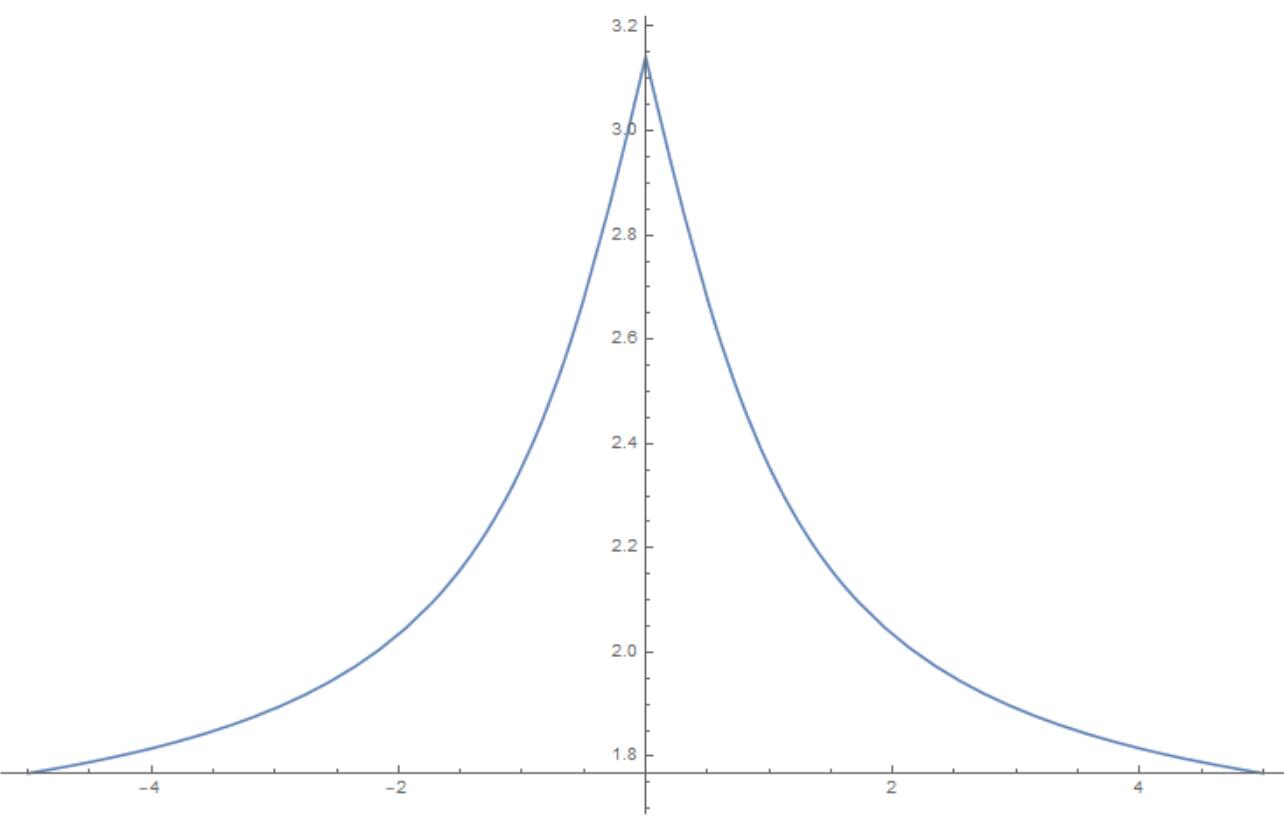}
\caption{$\pi-\arccos(\frac{1}{\sqrt{1+x^2}})$}\label{fig-2}
\end{minipage}
\begin{minipage}[t]{0.49\linewidth}
\centering
\includegraphics[scale=0.20]{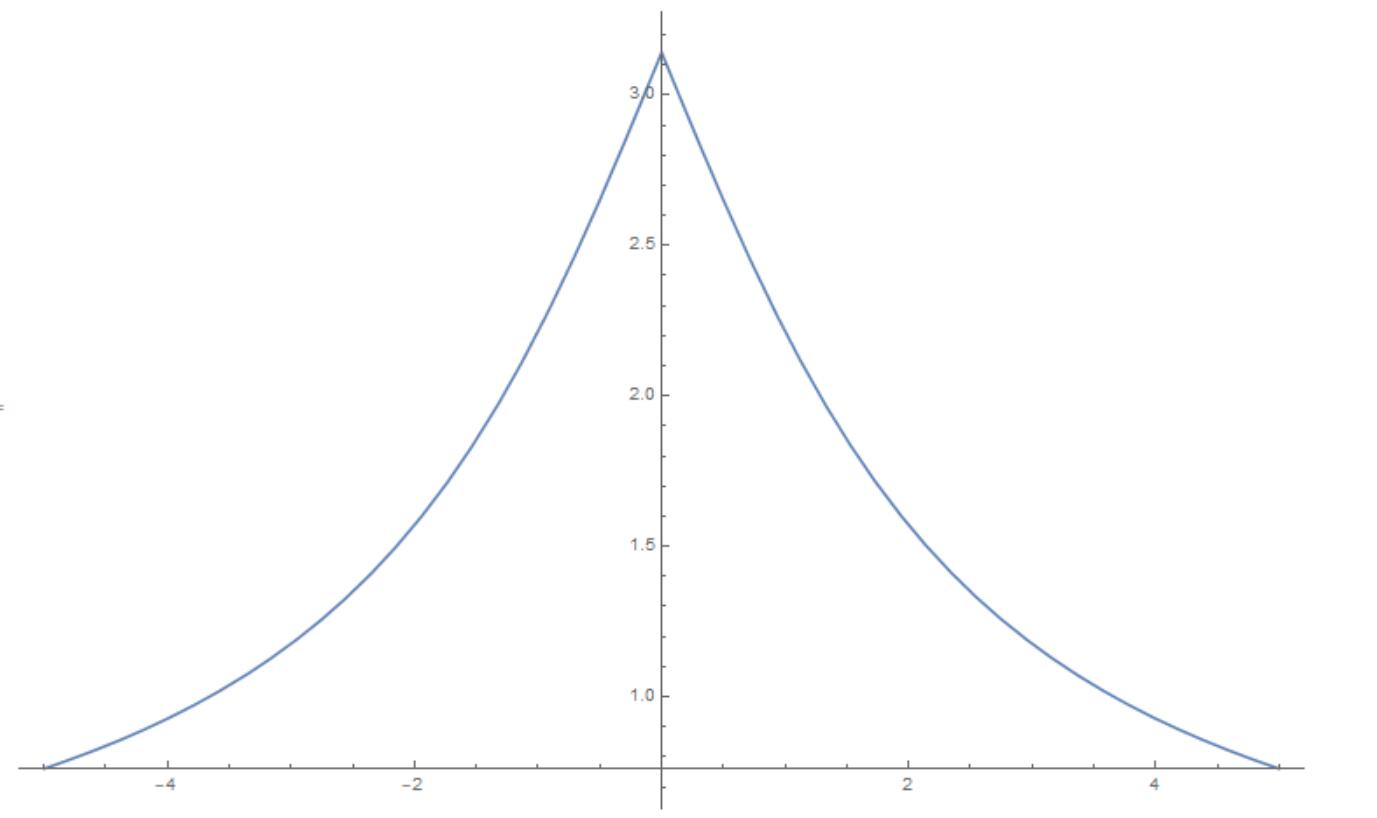}
\caption{$\pi-\arccos(\frac{4-x^2}{4+x^2})$}\label{fig-3}
\end{minipage}
%
\end{figure}

For the general DNN, we can obtain similar results.
\begin{theorem}
\label{th-1}
Let $\F$ be a DNN defined in \eqref{eq-dnn1} and $n_{j}=n$ for all $j\in\{1,2,\dots,L-1\}$.
If the $L_{2,\infty}$ norm of all weight matrices of $\F$ is smaller than
$c$, then the dihedral angle between two adjacent faces of $\F_g$  is bigger than $$\pi-\arccos(\frac{1}{\sqrt{1+c^{L} n^{\frac{L-2}{2}}}})
\hbox{ {\rm{or}} }
\pi-\arccos(\frac{4-c^{L} n^{\frac{L-2}{2}}}{4+c^{L} n^{\frac{L-2}{2}}}).$$
\end{theorem}
\begin{proof}
Let  $A,  B$  be two adjacent faces of $\F_{g}$,  and $R_{A}$,  $R_{B}$ their linear regions, respectively. We assume that $\F$ satisfies Assumption \ref{ass-p}. By Lemma \ref{g2}, there exists a unique activation function which changes its sign over $R_{A}\cup R_{B}$. We assume that {\em this happens on the $k$-th layer} where $1\le k\le L-1$.
Then the activation function on the first $k-1$ layer and the last $k+1$ layers have the same sign on $R_{A}\cup R_{B}$, that is,
\begin{equation}
\label{eq-kk}
   x_{k-1}=U_{k,1}x_{0}+e_{k,1},\
   x_{k}=\Relu(W_{k}x_{k-1}+b^{k}),\
   \F(x_{0})=U_{k,2}x_{k}+e_{k,2}
\end{equation}
for all $x_{0}$ in $R_{A}\cup R_{B}$, where
$U_{k,1}\in \R^{n_{k-1}\times n_{0}}$,
$e_{k,1}\in \R^{n_{k-1}}$, $U_{k,2}\in \R^{1\times n_{k}}$,  $e_{k,2}\in \R$.
Then we have
\begin{equation}
\label{eq-kk2}
   x_{k}=\Relu(W_{k}x_{k-1}+b^{k})=\Relu(W_k(U_{k,1}x_{0}+e_{k,1})+b^k)=\Relu(U\cdot x_0+\hat{b}^k),
\end{equation}
where $U=W_kU_{k,1}$ and $\widehat{b}_k=b_k+W_ke_{k,1}$.
%
Let $k$ be defined in \eqref{eq-kk}.
We first treat $\F$ as a one-hidden-layer network whose input is $x_{0}$, and can be computed as in \eqref{eq-kk2}.  Then because of Lemma \ref{l11},  we know the $L_{2,\infty}$ norm of $U_{k,2}$ is smaller than $c^{L-k}n^{\frac{L-k-1}{2}}$, and the $L_{2,\infty}$ norm of $U$ is smaller than $c^{k} n^{\frac{k-1}{2}}$.
Then, the theorem can be proved similarly to Lemma \ref{d1}.
\end{proof}

\begin{cor}
If $U_{k,1}$ in \eqref{eq-kk} is the multiple of an orthogonal matrix, then the dihedral angle between $A$ and $B$ is bigger than $$\pi-\arccos\frac{1}{\sqrt{1+c^{L-k+1}\times n^{\frac{L-k-1}{2}}}}
\hbox{ {\rm{or}} }
\pi-\arccos\frac{4-c^{L-k+1}\times n^{\frac{L-k-1}{2}}}{4+c^{L-k+1}\times n^{\frac{L-k-1}{2}}}.$$
\end{cor}
\begin{proof}
Let $k$ be defined in \eqref{eq-kk}.
We first treat $\F$ as a one-hidden-layer network whose input is $x_{k-1}$, and can be computed as in \eqref{eq-kk}. From Lemma \ref{l11},   the $L_{2,\infty}$ norm of $U_{k,2}$ is smaller than $c^{L-k} n^{\frac{L-k-1}{2}}$ and the $L_{2,\infty}$ norm of $W_k$ is smaller than $c$.
On the other hand, when $U_{k,1}$ in \eqref{eq-kk} is the multiple of an orthogonal matrix, it  does not change the angle.
Finally, the result can be proved similarly to Lemma \ref{d1}.
\end{proof}

\section{$L_{2, \infty}$ normalization  and robustness of DNN}
\label{sec-4}
In this section, we show that the $L_{2, \infty}$ normalization leads
to more robust DNNs by proving a lower bound of the robust region of the DNN.

We assume that $\F$ is a {\em classification DNN}, that is, its output
values $\widehat{\F}$ are discrete.
To simplify the discussion, we assume that $n_L=2$ and if the first coordinate of $\F(x)$ is bigger than the second one, then $\widehat{\F}(x)=0$ and $1$ otherwise.
%
%

Let $\theta$ be a continuous open set in $\R^{n_{0}}$. $\F$ is said to be
{\em robust} over $\theta$ if $\widehat{\F}$ gives the same label for all $x\in \theta$.
We also say that $\theta$ is a {\em robust region} of $\F$.

Denote $B_{x,r}$ to be the spherical ball with center $x$ and radius $r$.
Since $\F$ is piecewise linear and continuous, for almost all input $x\in\R^{n_0}$,
$B_{x,\epsilon}$ is a robust region of $\F$ if $\epsilon$ is small enough.
%
%
For any $x \in \R^{n_{0}}$, let $r_x$ be the maximum value such that
$B_{x,r_x}$ is a robust region of $\F$. Then  the volume of $B_{x,r_x}$ is $$A_{x}=C_{n_{0}}r^{n_{0}}_{x}$$
where $C_{n}=\frac{\pi^{n}}{\Gamma(n/2-1)}$ and $\Gamma(x)=\int_{0}^{\infty}t^{x-1}e^{-t}dt$ for all $n\in Z_{+}$. Note that $C_{n}$ just depends on $n$.
%
We now define a measure for the robustness of $\F$.
\begin{define}
\label{e7}
Let $\S$ be the training set of network $\F$, $|\S|=m$, and $\S=\{(x_{i},y_i)\,|\, i=1,\ldots,m \}$. We set $G_{\F,\S}=\{i|\widehat{F}(x_{i})=L(x_{i})\}$, where $L(x_{i})$ is the label of $x_{i}$. Then  the {\em robust volume}
and the {\em robust radius} of $\F$ are defined respectively as
$$V_{\F,\S}=\frac{\sum_{i\in G_{\F,\S}}A_{x_{i}}}{m}
\,\,\hbox{\rm  and }\,
r_{\F,\S}=\frac{\sum_{i\in G_{\F,\S}}r_{x_{i}}}{m}.$$
\end{define}

Since $\F$ is a classification DNN, most of the training data should be in $G_{\F,\S}$,
so the above definition measures for the robustness of the DNN in certain sense.
$V_{\F,\S}$ is the average volume of maximal robust spheres with the sample points as centers and $r_{\F,\S}$ is the average radius of such spheres.
%
%
%
We will show that the  $L_{2,\infty}$ normalization will lead to a lower bound for $V_{\F,\S}$ and $r_{\F,\S}$.

We first consider a simple network $\H: \R^{n}\rightarrow \R^{2}$, which has one hidden layer with dimension $n$, with activation function  \Relu, and output  $0$ and $1$.  Then 
\begin{equation}
\label{dnn-h2}
\H(x)=L\cdot \Relu(Wx+b)+d
\end{equation}
where ${W, L}$ are the weight matrices and ${b, d}$ are the bias vectors.
If the first coordinate $\H(x)$ is bigger than the second one,
the DNN outputs $\widehat{\H}(x)=0$ and $1$ otherwise.

\begin{lemma}
\label{p-l2r}
For the DNN $\H$ defined in \eqref{dnn-h2}, if the $L_{2, \infty}$ norm of the weight matrices of $\H$ is smaller than $c$, then $||\H(x+\alpha) -\H(x)||_{\infty} \le c^2||\alpha||_{2}$ for $x,\alpha\in\R^n$.
\end{lemma}
\begin{proof}
We have $\H(x)=L\cdot \Relu(Wx+b)+d$ and $\H(x+\alpha)=L\cdot \Relu(Wx+W\alpha+b)+d$.
Then   $|\H(x+\alpha) -\H(x)| \leq  |L\cdot (W\alpha)|$.
Because of the $L_{2, \infty}$ norm constraint, we obtain
$||\H(x+\alpha) -\H(x)||_{\infty} \le c^2||\alpha||_{2}$.
\end{proof}

We assume that $\S$ is the training set of $\H$, and the loss function is square norm or crossentropy (when we choose crossentropy, the output layer uses activation function Softmax). Then we have the result below.

\begin{lemma}
\label{p-l2ri}
Assume that the loss function   is the square norm, 
the value of the loss function on $\S$ is smaller than $\epsilon$, 
the accuracy of $\H$ on the training set is bigger than $\gamma$, and $\gamma>2\epsilon$.
If the $L_{2, \infty}$ norm of the weight matrices is smaller than $c$, then we have
$$V_{\H,\S}\ge C_{n}\frac{(\gamma-\sqrt{2\epsilon\gamma})^n}{2^n\cdot c^{2n}\cdot \gamma^{n-1}}
\,\hbox{\rm  and }
r_{\H}\ge \frac{\gamma-\sqrt{2\epsilon\gamma}}{2c^{2}}.$$
\end{lemma}
\begin{proof}
Let the training set be $\S=\{(x_{i},y_i)\,|\, i=1,\ldots,m \}$, 
where $y_{i}=[1,0]$ when the label of $x_i$ is $0$, 
and $y_{i}=[0,1]$ when the label of $x_i$ is $1$.
Since the loss function is square norm, we have
$\Loss=\frac{1}{m}\sum^m_{i=1}[({\H}_0(x_i)-y_{i,0})^2+({\H}_1(x_i)-y_{i,1})^2]$,
where $y_{i,0}$ ($y_{i,1}$) is the first (second) component of $y_i$.
We can also write it as $\Loss=\frac{1}{m}\sum_{i=1}^{m}[(1-\H_{c_i}({x_{i}}))^2+(\H_{f_i}(x_{i}))^2]$, where
$c_i$  is the label of $x_i$ and $f_i=1-c_i$.
%

Since the value of loss function is smaller than $\epsilon$, we have
$$\sum_{i\in G_{\H,\S}}[(1-\H_{c_i}({x_{i}}))^2+(\H_{f_i}(x_{i}))^2]\le \sum_{i=1}^m[(1-\H_{c_i}({x_{i}}))^2+(\H_{f_i}(x_{i}))^2]\le m\epsilon$$
By Lemma \ref{p-l2r}, for any $x_{i}\in G_{\H,\S}$, we have $r_{x_{i}}\ge \frac{\H_{c_i}(x_{i})-\H_{f_i}(x_{i})}{2c^2}$.
Then
$$V_{\H,\S}\ge C_{n}\frac{\sum_{i\in G_{\H,\S}}(r_{x_{i}}^n)}{m}\ge C_{n}\frac{\sum_{i\in G_{\H,\S}}(\H_{c_i}(x_{i})-\H_{f_i}(x_{i}))^n}{2^n\cdot m\cdot c^{2n}}$$
and
$$r_{\H,\S}\ge \frac{\sum_{i\in G_{\H,\S}}r_{x_{i}}}{m}\ge \frac{\sum_{i\in G_{\H,\S}}\H_{c_i}(x_{i})-\H_{f_i}(x_{i})}{2m\cdot c^{2}}.$$
Due to the inequality $(1-a)^2+b^2\ge \frac{(1-a+b)^2}{2}$ for $a\in \R$ and $b\in \R$,  we have
\begin{equation}
\label{eq-421}
\sum_{i\in G_{\H,\S}}(1-\H_{c_i}({x_{i}})+\H_{f_i}(x_{i}))^2/2 \le \sum_{i\in G_{\H,\S}}(1-\H_{c_i}({x_{i}}))^2+(\H_{f_i}(x_{i}))^2\le m\epsilon.
\end{equation}

Let $T_{i}=\H_{c_i}({x_{i}})-\H_{f_i}(x_{i})$. It is easy to see that when $i\in G_{\H,\S}$, $T_{i}>0$. Since the accuracy on $\S$ is bigger than $\gamma$, we have $|G_{\H,\S}|\ge m\gamma$. Using the inequality $(\sum_{i=1}^{k}a_{1})^p\le k^{p-1}(\sum_{i=1}^{k}(a_{i}^p))$ ($a_{i}\in \R_{+}$ and $p\ge1$) to \eqref{eq-421}, we have
$$\frac{(m\gamma-\sum_{i\in G_{\H,\S}}T_{i})^2}{m\gamma}\le \sum_{i\in G_{\H,\S}}(1-\H_{c_i}({x_{i}})+\H_{f_i}(x_{i}))^2 \le 2m\epsilon.$$
That is $$\sum_{i\in G_{\H,\S}}T_{i}\ge m(\gamma-\sqrt{2\epsilon\gamma}).$$
So we have
\begin{eqnarray*}
V_{\H,\S}
&& \ge C_{n}\frac{\sum_{i\in G_{\H,\S}}(T_{i})^n}{2^n\cdot m\cdot c^{2n}}
 \ge C_{n}\frac{(\sum_{i\in G_{\H,\S}}T_{i})^n}{2^n\cdot m\cdot c^{2n}\cdot(m\gamma)^{n-1}}\\
&& \ge C_{n}\frac{m^n(\gamma-\sqrt{2\epsilon\gamma})^n}{2^n\cdot m\cdot c^{2n}\cdot(m\gamma)^{n-1}}
 =C_{n}\frac{(\gamma-\sqrt{2\epsilon\gamma})^n}{2^n\cdot c^{2n}\cdot \gamma^{n-1}}.
\end{eqnarray*}
and
$$r_{\H,\S}\ge\frac{\sum_{i\in G_{\H,\S}}\H_{c_i}(x_{i})-\H_{f_i}(x_{i})}{2m\cdot c^{2}}=\frac{\sum_{i\in G_{\H,\S}}(T_{i})}{2m\cdot c^{2}}\ge\frac{\gamma-\sqrt{2\epsilon\gamma}}{2\cdot c^{2}}.$$
The lemma is proved.
\end{proof}

\begin{lemma}
\label{p-l2ti}
Assume that the loss function   is the cross entropy,
the value of the loss function on $\S$ is smaller than $\epsilon$,
the accuracy of $\H$ on the training set is bigger than $\gamma$, and $\ln2\cdot \gamma>\epsilon $.
If the $L_{2, \infty}$ norm of the weight matrix of $H$ is smaller than $c$, then we have
  $$V_{\H,\S}\ge C_{n}\frac{(\ln2\cdot\gamma-\epsilon)^n}{c^{2n}\cdot \gamma^{n-1}}
  \,\hbox{ and }
   r_{\H,\S}\ge\frac{\ln2\cdot \gamma-\epsilon}{c^{2}}.$$
\end{lemma}
\begin{proof}
Let $\H^1$ be a DNN which  has the same parameters and structure with that of $\H$, except the activation function of output layer is Softmax.
Since the value of the loss function is smaller than $\epsilon$, we have
$$\sum_{i\in G_{\H^1,\S}}-\ln\frac{e^{\H^1_{c_i}(x_{i})}}{e^{\H^1_{c_i}(x_{i})}+e^{\H^1_{f_i}(x_{i})}}\le m\epsilon$$where $c_i$ is the label of $x_{i}$, $f_i=1-c_i$ .
By Lemma \ref{p-l2r},  for any ${i}\in G_{\H,\S}$, we have $r_{x_{i}}\ge \frac{\H^1_{c_i}(x_{i})-\H^1_{f_i}(x_{i})}{2c^2}$. Then  
$$V_{\H,\S}\ge C_{n}\frac{\sum_{i\in G_{\H,\S}}(r_{x_{i}}^n)}{m}\ge C_{n}\frac{\sum_{i\in G_{\H,\S}}(\H^1_{c_i}(x_{i})-\H^1_{f_i}(x_{i}))^n}{2^n\cdot m\cdot c^{2n}}$$
and
$$r_{\H,\S}\ge \frac{\sum_{i\in G_{\H,\S}}r_{x_{i}}}{m}\ge \frac{\sum_{i\in G_{\H,\S}}\H^1_{c_i}(x_{i})-\H^1_{f_i}(x_{i})}{2m\cdot c^{2}}.$$

Let $T_{i}=\H^1_{c_i}({x_{i}})-\H^1_{f_i}(x_{i})$. It is easy to see that when $i\in G_{\H,\S}$, $T_{i}>0$, so we obtain
$$\sum_{i\in G_{\H,\S}}-\ln\frac{e^{H^1_{c_i}(x_{i})}}{e^{H^1_{c_i}(x_{i})}+e^{H^1_{f_i}(x_{i})}}=\sum_{i\in G_{\H,\S}}-\ln\frac{e^{T_{i}}}{e^{T_{i}}+1}=\sum_{i\in G_{\H,\S}}\ln(1+\frac{1}{e^{T_{i}}})\le m\epsilon.$$

We need a simple inequality: $\ln(1+e^{-x})\ge \ln2-0.5\cdot x$. Taking the natural exponential at two sides of the inequality, it becomes $1+e^{-x}\ge 2e^{-0.5x}$ which is obvious.
Because of the above inequality, we have
$$\sum_{i\in G_{\H,\S}}(\ln2-0.5T_i)\le\sum_{i\in G_{\H,\S}}\ln(1+\frac{1}{e^{T_{i}}})\le  m\epsilon.$$
Since the accuracy on $\S$ is bigger than $\gamma$, that is $|G_{\H,\S}|\ge m\gamma$, we have
$$\ln2\cdot m\gamma-0.5\sum_{i\in G_{\H,\S}}T_{i}\le\sum_{i\in G_{\H,\S}}(\ln2-0.5x)\le m\epsilon.$$
That is $$\sum_{i\in G_{\H,\S}}T_{i}\ge 2m(\ln2\cdot \gamma-\epsilon).$$
So we have
\begin{eqnarray*}
V_{\H,\S}
&&\ge C_{n}\frac{\sum_{i\in G_{\H,\S}}(T_{i})^n}{2^n\cdot m\cdot c^{2n}}
   \ge C_{n}\frac{(\sum_{i\in G_{\H,\S}}T_{i})^n}{2^n\cdot m\cdot c^{2n}\cdot(m\gamma)^{n-1}}\\
&&\ge C_{n}\frac{m^n(\ln2\cdot\gamma-\epsilon)^n}{m\cdot c^{2n}\cdot(m\gamma)^{n-1}}=C_{n}\frac{(\ln2\cdot\gamma-\epsilon)^n}{c^{2n}\cdot \gamma^{n-1}}
\end{eqnarray*}
and
$$r_{\H,\S}\ge\frac{\sum_{i\in G_{\H,\S}}\H_{c_i}(x_{i})-\H_{f_i}(x_{i})}{2m\cdot c^{2}}=\frac{\sum_{i\in G_{\H,\S}}(T_{i})}{2m\cdot c^{2}}\ge\frac{\ln2\cdot \gamma-\epsilon}{c^{2}}.$$
The lemma is proved.
\end{proof}

\begin{exam}
\label{p-l2rii}
Let $S$ be the set of images in MNIST whose labels are $0$ or $1$. $\F$ is a DNN with parameters $L=2$, $n_0=n_1=784$, $n_2=2$. The activation function of the hidden layer is $\Relu$ and that of the output layer is {\em Softmax}. The loss function is crossentropy.
We train $\F$ on $S$ with $L_{2,\infty}$ normalization and obtain the following result.
From the table, we can see that the robustness radius is reasonably good.
\begin{table}[H]
\centering
\begin{tabular}{|c|c|c|c|}
  \hline
  $L_{2,\infty}$ norm & Accuracy on $\S$ & Loss function & Robustness radius\\
  0.2 & 99.96$\%$ & 0.3097 &9.579 \\
  0.3 & 99.94$\%$ & 0.2945 &4.424 \\
  \hline
\end{tabular}
\caption{Average robustness radius}
\label{tab-e1}
\end{table}
\end{exam}

For a general DNN $\F$, we can calculate $V_{\F,\S}$ and $r_{\F,\S}$ by the same way.
We assume that $\S=\{(x_{i},y_i)\,|\, i=1,\ldots,m \}$ is the training set of $\F$ and the loss function is square norm or cross entropy(when we choose cross entropy, the output layer uses activation function Softmax).
\begin{lemma}
\label{p-l2r2}
Let $\F$ be the DNN defined in \eqref{eq-dnn0}, $n_{i}=n$ for $0\le i \le L-1$, and $n_{L}=2$.
If the $L_{2,\infty}$ norm of the weight matrices of $F$ is smaller than $c$,
then for a small noise $\epsilon$ to the input $X$,  we have
$||\F(X+\epsilon)-\F(x)||_{\infty}\le n^{L/2-1}c^{L}||\epsilon||_{2}$.
\end{lemma}
\begin{proof}
Treat each layer of $\F$ as a one-layer network for the classification problem.
When a noise $\epsilon$ is added to the input, the change of the output of the first layer is smaller than $c||\epsilon||_{2}$ for every component, so the change of the square norm is smaller than $\sqrt{n}c||\epsilon||_{2}$ which can  be considered as the noise added to the input of the second layer. Repeat the procedure, we obtain the result.
\end{proof}

\begin{theorem}
\label{p-l5r}
Assume that the value of the loss function (square norm) of $\F$ on $\S$ is smaller than $\epsilon$, the accuracy on the training set is bigger than $\gamma$, and $\gamma>2\epsilon $. If the $L_{2, \infty}$ norm of the weight matrices of $\F$ is smaller than $c$, then
$$V_{\F,\S}\ge C_{n}\frac{(\gamma-\sqrt{2\epsilon\gamma})^n}{2^n\cdot (n^{L/2-1}c^L)^{n}\cdot \gamma^{n-1}} \,\hbox{ and }r_{\F,\S}\ge \frac{\gamma-\sqrt{2\epsilon\gamma}}{2n^{L/2-1}c^L}.$$
\end{theorem}
\begin{proof}
It can be proved similarly to Lemma \ref{p-l2ri}.
\end{proof}

\begin{theorem}
\label{p-l5ir}
Assume that the value of the loss function (cross entropy) of $\F$ on $\S$ is smaller than $\epsilon$, the accuracy on the training set is bigger than $\gamma$, and $\ln2 \cdot\gamma>\epsilon $. If the $L_{2, \infty}$ norm of the weight matrices of $\F$ is smaller than $c$, then
$$V_{\F,\S}\ge C_{n}\frac{(\ln2\cdot\gamma-\epsilon)^n}{(n^{L/2-1}c^L)^{n}\cdot \gamma^{n-1}}
\,\hbox{ and }
r_{\F,\S}\ge \frac{\ln2\cdot\gamma-\epsilon}{n^{L/2-1}c^L}.$$
\end{theorem}
\begin{proof}
It can be proved similarly to Lemma \ref{p-l2ti}.
\end{proof}

\section{Rademacher complexity of DNN with $L_{2,\infty}$ normalization}
\label{sec-5}
In this section, we will give an upper bound for the
Rademacher complexity of DNNs with $L_{2,\infty}$ normalization.
A {\em Rademacher random variable} is a random variable $\lambda$ which satisfies $P(\lambda=1)=P(\lambda=-1)=0.5$.

\begin{define}
Let $K=\{f:\R^{n}\to \R\}$ be a class of functions or the {\em hypothesis space}, and $D=\{(x_{i},y_i)\,|\, i=1,\ldots,m \}$   the training set, where $x_{i}\in \R^{n}$ and $y_i \in \R$.
The Rademacher complexity of $K$ on $D$ is
 $$R_D(K)=\E_{\lambda}[\sup_{f\in K}\frac{1}{m}\sum_{i=1}^{m}\lambda_if(x_{i})]$$
where $\lambda=\{\lambda_{i}\,|\, i=1,\ldots,m \}$ is a set of $m$ independent Rademacher random variables.
\end{define}

When the Rademacher complexity is small,  the complexity of the hypothesis space is  simple.
For instance, when there is only one function in the hypothesis space $K$, the Rademacher complexity $R_{D}(K)$ is always $0$. The following result is obvious.
\begin{lemma}
\label{cs1}
When the hypothesis space $K$ consists of constant functions whose values are in $[-c,c]$ for $c>0$, the Rademacher complexity $R_{D}(K)$ is smaller than $c$.
\end{lemma}

Now we will compute an upper bound of the Rademacher complexity of the hypothesis space which is the set of DNNs with $L_{2,\infty}$ normalization.
\begin{define}
For  $d,n\in \Z_+$ and $c,b\in\R_{>0}$,
let $\Nn^{d,n}_{c,b}$ be the hypothesis space which contains all the DNNs $\F$ satisfying:
$\F$  has $d-1$ hidden layers;
every layer of $\F$  except the output layer has $n$ nodes;
the output layer has just one node and does not have activity function;
the $L_{2,\infty}$ norm of every weight matrix is smaller than $c$;
the $L_{\infty}$ norm of every bias vector is smaller than $b$;
and the activity function of every layer except the output layer is $\Relu$.
\end{define}
Let $\S=\{(x_i,y_i)\,|\, i=1,\ldots,m \}$ be a training set of $\Nn^{d,n}_{c,b}$. We will compute the Rademacher complexity of $\Nn^{d,n}_{c,b}$ on $\S$.
We first give several lemmas.
\begin{lemma}[\cite{N2014}]
\label{lcc1}
$R_{\Nn^{1,n}_{c,0}}(S)\le c\sqrt{\frac{2}{m}}\max_{i}||x_{i}||_{2}$.
\end{lemma}
We  extend Lemma \ref{lcc1} to include the bias vector.
\begin{lemma}
\label{lcc2}
$R_{\Nn^{1,n}_{c,b}}(S)\le c\sqrt{\frac{2}{m}}\max_{i}||x_{i}||_{2}+b$.
\end{lemma}
\begin{proof} Since $d=1$, we have $f(x) = Wx+b_{0}$.
Then, we have
\begin{align*}
&R_{\Nn^{1,n}_{c,b}}(S)\\
&=\E_{\lambda}[\frac{1}{m}\sup_{f\in\Nn^{1,n}_{c,b}}\sum_{i=1}^{m}\lambda_{i}f(x_i)]\\
&=\E_{\lambda}[\frac{1}{m}\sup_{||W||_2\le c, |b_0|\le b}\sum_{i=1}^{m}\lambda_{i}(Wx_i+b_{0})]\\
 &\le\E_{\lambda}[\frac{1}{m}\sup_{||W||_2\le c}
 \sum_{i=1}^{m}\lambda_{i}(Wx_i)]+\E_{\lambda}[\frac{1}{m}\sup_{|b_0|\le b}\sum_{i=1}^{m}\lambda_{i}b_0]\\
&\le R_{\Nn^{1,n}_{c,0}}(S)+b\\
&\le c\sqrt{\frac{2}{m}}\max_{i}||x_{i}||_{2}+b.
\end{align*}
The last inequality follows from Lemma \ref{lcc1}.
\end{proof}

\begin{lemma}
\label{lrz}
Let $f: \R^n\to \R^n$ be a function, $d\ge2$, $\lambda=\{\lambda_i\,|\, i=1,\ldots,m \}$ a set of Rademacher random variables, $\{x_{i}\,|\, i=1,\ldots,m \}\subset \R^n$. Then we have $$\sup_{||W||_{2,\infty}\le a_{1},||B||_{\infty}\le a_2}||\sum^{m}_{i=1}\lambda_{i}\sigma(W\sigma(f(x_i))+B)||_{2}=\sqrt{n}\sup_{||w||_2\le a_1,||b||_{\infty}\le a_2}|\sum_{i=1}^{m}\lambda_i\sigma(w\sigma(f(x_i))+b)|$$
where $W\in \R^{n\cdot n}$, $B\in \R^{n}$, $w^t\in \R^n$, and $b \in \R$.
\end{lemma}
\begin{proof} We have
\begin{align*}
&\sup_{||W||_{2,\infty}\le a_{1},||B||_{\infty}\le a_2}||\sum^{m}_{i=1}\lambda_{i}\sigma(W\sigma(f(x_i))+B)||_{2}\\
&=\sup_{||W||_{2,\infty}\le a_{1},||B||_{\infty}\le a_2}\sqrt{\sum_{j=1}^{n}(\sum_{i=1}^{m}(\lambda_i\sigma(W_{j}\sigma(f(x_i))+B_{j})))^2}\\
&=\sup_{||W_j||_{2}\le a_{1},||B_j||_{\infty}\le a_2}\sqrt{\sum_{j=1}^{n}(\sum_{i=1}^{m}(\lambda_i\sigma(W_{j}\sigma(f(x_i))+B_{j})))^2}\\
&=\sup_{||w||_{2,\infty}\le a_{1},||b||_{\infty}\le a_2}\sqrt{n}\sqrt{(\sum_{i=1}^{m}(\lambda_i\sigma(W_{j}\sigma(f(x_i))+B_{j})))^2}\\
&=\sqrt{n}\sup_{||w||_2\le a_1,||b||_{\infty}\le a_2}|\sum_{i=1}^{m}\lambda_i\sigma(w\sigma(f(x_i))+b)|.
\end{align*}
The lemma is proved.
\end{proof}

The following result is well known.
\begin{lemma}[Contraction Lemma]
\label{clem}
 $\phi:\R\to \R$ is a Lipschitz continuous function with a constant $L$ and $\phi(0)=0$. Then for any class $F$ of functions mapping from $X$ to $\R$, and any $S=\{x_i\,|\, i=1,\ldots,m \}$, we have $$\E_{\lambda\in \{-1,-1\}^m}[\frac{1}{m}\sup_{f\in F}|\sum_{i=1}^{m}\lambda_i\phi(f(x_i))|]\le2L\E_{\lambda\in \{-1,-1\}^m}[\frac{1}{m}\sup_{f\in F}|\sum_{i=1}^{m}\lambda_if(x_i)|].$$
\end{lemma}

We now give an upper bound for the Rademacher complexity of $\Nn^{d,n}_{c,b}$.
\begin{theorem} The Rademacher complexity of $\Nn^{d,n}_{c,b}$ on the training set
$S=\{x_i,y_i\,|\, i=1,\ldots,m \}$ satisfies
$$R_{\Nn^{d,n}_{c,b}}(S)\le c^d\sqrt{n}^{d-1}\sqrt{\frac{2}{m}}\max_{i}||x_{i}||_{2}+b\frac{(c\sqrt{n})^d-1}{c\sqrt{n}-1}.$$
\end{theorem}
\begin{proof}
For  $f\in \Nn^{d,n}_{c,b}$, let $W_{l}$ be the weight matrix of $l$-th layer,
 $b_{l}$  the bias vector of $l$-th layer,
and $x_{l}$   the output of $l$-th layer of $f(x)$.
Denote $f_{l}(x)=W_{l}x_{l-1}+b_{l}$. Then $\sigma(f_{l}(x))=x_{l}$ for $l<d$, where $\sigma $ is $\Relu$.
Since the output layer has just one node and does not have activity function,
we have $f(x_{i}) = W^d\sigma(f_{d-1}(x_{i}))+b_d$. Then
\begin{align*}
&\E_{\lambda}[\sup_{f\in\Nn^{d,n}_{c,b}}\frac{1}{m}\sum_{i=1}^{m}\lambda_{i}f(x_{i})]\\
&=\E_{\lambda}[\sup_{f\in\Nn^{d,n}_{c,b}}\frac{1}{m}\sum_{i=1}^{m}\lambda_{i}(W^d\sigma(f_{d-1}(x_{i}))+b_d)]\\
&\le\E_{\lambda}[\sup_{f\in\Nn^{d,n}_{c,b}}\frac{1}{m}\sum_{i=1}^{m}\lambda_{i}(W^d\sigma(f_{d-1}(x_{i})))]+
\E_{\lambda}[\sup_{f\in\Nn^{d,n}_{c,b}}\frac{1}{m}\sum_{i=1}^{m}\lambda_{i}b_d]\\
&\le\E_{\lambda}[\sup_{f\in\Nn^{d,n}_{c,b}}\frac{1}{m}\sum_{i=1}^{m}\lambda_{i}(W^d\sigma(f_{d-1}(x_{i})))]+
b.
\end{align*}
The last inequality follows from  Lemma \ref{cs1}. Now we just need to consider the first part.
\begin{align*}
&\E_{\lambda}[\sup_{f\in\Nn^{d,n}_{c,b}}\frac{1}{m}\sum_{i=1}^{m}\lambda_{i}(W^d\sigma(f_{d-1}(x_{i})))]\\
&\le\E_{\lambda}[\sup_{f\in\Nn^{d,n}_{c,b}}\frac{c}{m}\cdot ||\sum^{m}_{i=1}\lambda_{i}\sigma(f_{d-1}(x_{i}))||_2]\\
&=\E_{\lambda}[\sup_{f\in\Nn^{d,n}_{c,b}}\frac{c}{m}\cdot ||\sum^{m}_{i=1}\lambda_{i}\sigma(W^{d-1}\sigma (f_{d-2}(x_{i}))+b^{d-1})||_2]\\
&=\E_{\lambda}[\sup_{f_{d-2},W^{d-1},b^{d-1}}\frac{c}{m}\cdot ||\sum^{m}_{i=1}\lambda_{i}\sigma(W^{d-1}\sigma (f_{d-2}(x_{i}))+b^{d-1})||_2]\\
&=\E_{\lambda}[\sup_{f_{d-2},||W||_2\le c,||b||_{\infty}\le b}\frac{c\sqrt{n}}{m}\cdot |\sum^{m}_{i=1}\lambda_{i}\sigma(W\sigma (f_{d-2}(x_{i}))+b)|].
\end{align*}
The last equality follows from Lemma \ref{lrz}. We can see that $W\sigma (f_{d-2}(x_{i}))+b$ is a function in $\Nn^{d-1,n}_{c,b}$. Then
\begin{align*}
&\E_{\lambda}[\sup_{f_{d-2},||W||_2\le c,||b||_{\infty}\le b}\frac{c\sqrt{n}}{m}\cdot |\sum^{m}_{i=1}\lambda_{i}\sigma(W\sigma (f_{d-2}(x_{i}))+b)|]\\
&=\E_{\lambda}[\sup_{f\in \Nn^{d-1,n}_{c,b}}\frac{c\sqrt{n}}{m}\cdot |\sum^{m}_{i=1}\lambda_{i}\sigma(f(x_i))|]\\
&\le2\frac{c\sqrt{n}}{m}\E_{\lambda}[\sup_{f\in \Nn^{d-1,n}_{c,b}} |\sum^{m}_{i=1}\lambda_{i}f(x_i)|].
\end{align*}
The last inequality follows from Lemma \ref{clem}, since $\sigma$ is
a Lipschitz continuous function with constant $1$. On the other hand,
$$\E_{\lambda}[\sup_{f\in \Nn^{d-1,n}_{c,b}} |\sum^{m}_{i=1}\lambda_{i}f(x_i)|]=\E_{\lambda}[\sup_{f\in \Nn^{d-1,n}_{c,b}} \sum^{m}_{i=1}\lambda_{i}f(x_i)].$$
So we have $$R_{\Nn^{d,n}_{c,b}}(S)\le c\sqrt{n}  R_{\Nn^{d-1,n}_{c,b}}(S)+b.$$
Then we have $$R_{\Nn^{d,n}_{c,b}}(S)\le(c\sqrt{n})^{d-1}R_{\Nn^{1,n}_{c,b}}(S)+b\frac{(c\sqrt{n})^{d-1}-1}{c\sqrt{n}-1}.$$
By Lemma \ref{lcc2}, we have
$$R_{\Nn^{d,n}_{c,b}}(S)\le c^d\sqrt{n}^{d-1}\sqrt{\frac{2}{m}}\max_{i}||x_{i}||_{2}+b \frac{(c\sqrt{n})^d-1}{c\sqrt{n}-1}.$$
The theorem is proved.
\end{proof}

When  $b=0$, we obtain Theorem 1 in \cite{N2014}:
 $$R_{\Nn^{d,n}_{c,0}}(S)\le c^d\sqrt{n}^{d-1}\sqrt{\frac{\min\{2,4\log 2n\}}{m}}\max_{i}||x_{i}||_{2}=c^d\sqrt{n}^{d-1}\sqrt{\frac{2}{m}}\max_{i}||x_{i}||_{2}.$$

\section{Algorithm and experimental results}
\label{sec-6}
\subsection{Algorithm with $L_{2, \infty}$ normalization}
In this section, we give an algorithm for training a DNN with $L_{2,\infty}$ normalization.
The algorithm modifies the standard gradient descent training algorithm as follows:
after doing each gradient descent, we set the $L_{2}$ norm  of
the rows of the weight matrices to be $c$ if they are larger than $c$,
which is a standard method used in~\cite{Y2018}.

\begin{algorithm}[H]
\caption{Algorithm for $L_{2, \infty}$ regularization}
\label{alg:Framwork}
\begin{algorithmic}
\REQUIRE ~~\\
The $L_{2, \infty}$ norm bound: $c$;\\
The set of training data:   $U$;\\
The initial value of $\Theta$:  $\Theta_{0}$

\ENSURE The trained parameters $\widetilde{\Theta}$.~~\\
%
In each iteration:\\
Input $\Theta_{k}=\{W_{l,k}, b_{l,k}\,|\, l=0,\ldots,L \}$\\
for all l in $\{1, 2, \dots, L\}$, do\\
\quad\quad\quad$W_{l,k+1}=W_{l,k}-\gamma_{k}\bigtriangledown \Loss(W_{l,k}, b_{l,k})$\\
\quad\quad\quad$b_{l,k+1}=b_{l,k}-\gamma_{k}\bigtriangledown \Loss(W_{l,k}, b_{l,k})$\\
$\gamma_{k}$ is the stepsize at iteration k. \\
$\quad$For all rows V in $W_{l,k+1}$\\
$\quad$$\quad$ if $||V||_{2}\ge c$ \\
$\quad$$\quad$$\quad$then  $V=\frac{c}{||V||_{2}}V$ \\
$\quad$$\quad$end if\\
$\quad$end for\\
end for\\
Output $\Theta_{k+1}=\{W_{l,k+1}, b_{l,k+1}\,|\, l=0,\ldots,L \}$\\
\end{algorithmic}
\end{algorithm}

When the algorithm terminates,  the $L_{2,\infty}$ norm of the weight matrices is smaller than or equal to $c$.
If $c$ is very small, then the algorithm might not terminate, so we need a validation set.
When the loss and accuracy on the validation set begin dropping, we stop the algorithm.

\subsection{Experimental results}
In this section, we give two sets of experimental results to show
how the $L_{2,\infty}$ normalization improves the robustness and
reduces the over-fitting of the DNN.

In the first experiment,
the network $\F$ has the structure: $L=3, n_0=n_1=n_2=784,  n_L=10$;
the output layer has activation function Softmax; and the loss function is crossentropy. We use MNIST to train the network to test the effect of the $L_{2,\infty}$ normalization
on the robustness.
We train six networks:

{\parskip=1pt
$\F_1$: the standard model~\eqref{eq-dnn0}.

$\F_2$: with $L_{1}$ regularization~\cite[sec.7.1]{DL} and super parameter\footnote{Here, we select the best super parameter among $10^{-k},k=1,2,3,4$.} 0.0001.

$\F_3$: with $L_{2,\infty}$ normalization $c=0.2$.

$\F_4$: with $L_{2,\infty}$ normalization $c=0.3$.

$\F_5$: with $L_{2,\infty}$ normalization $c=0.4$.

$\F_6$: with $L_{2,\infty}$ normalization $c=0.5$.
}

To test the robustness of the networks, we add five levels of noises to the test set.
The $i$-th level noise has the distribution $\eta_i \Nn$, where  $\Nn$ is standard normal distribution and $\eta_i=0.1i,i=1,\ldots,5$.
The pictures of the test sets with noises can be found in the appendix.
We then use the six networks to the five test sets
and results are given in Table \ref{tab-1}.
The  column {\em no noise} gives the accuracy on the raw data set,
and the column {\em noise $i$} gives the accuracy on the data with the $i$-th level noise.

\begin{table}[H]
\centering
\begin{tabular}{|c|c|c|c|c|c|c|}
  \hline
         & no noise & noise 1 & noise 2 & noise 3 & noise 4 & noise 5 \\
  $\F_1$ & 96.88$\%$& 96.09$\%$  & 95.23$\%$  & 93.57$\%$  & 90.09$\%$  & 85.35$\%$ \\
  $\F_2$ & 97.64$\%$& 97.19$\%$  & 96.63$\%$  & 95.97$\%$  & 90.98$\%$  &  86.67$\%$\\
  $\F_3$ & 96.89$\%$& 96.61$\%$  & 95.81$\%$  & 94.37$\%$  &91.37$\%$   & 87.02$\%$ \\
  $\F_4$ & 97.50$\%$& 97.03$\%$  & 96.78$\%$  & 95.25$\%$  &91.46$\%$   &  86.22$\%$\\
  $\F_5$ & 97.91$\%$& 97.61$\%$  & 96.97$\%$  & 95.66$\%$  &90.78$\%$   &   85.87$\%$\\
  $\F_6$ & 97.79$\%$& 97.49$\%$  &96.52$\%$   & 94.29$\%$  & 90.23$\%$  & 85.62$\%$ \\
  \hline
\end{tabular}
\caption{Accuracy of MNIST}
\label{tab-1}
\end{table}

From Table \ref{tab-1}, we can see that the $L_{2,\infty}$ normalization does enhance the robustness and accuracy.
%
For all levels of noise, the network with $L_{2,\infty}$ normalization
has better accuracy than the standard model, and the best
network with $L_{2,\infty}$ normalization is also better than
the network with $L_1$ regulation.
We can also see that, when the noise becomes big, the network with
smaller $L_{2,\infty}$ norm usually performs better.
The network achieves the best accuracy for each level of noise is give in Table \ref{tab-2}.
\begin{table}[H]
\centering
\begin{tabular}{|c|c|c|c|c|c|c|}
  \hline
   level   &no noise& noise 1 & noise 2  &noise 3& noise 4& noise 5 \\
   network & $\F_5$  & $\F_5$   &  $\F_5$   & $\F_2$ & $\F_4$  & $\F_3$ \\
  \hline
\end{tabular}
\caption{Network with the best accuracy}
\label{tab-2}
\end{table}

In the second experiment,
the network $\F$ has the structure: $L=3, n_0=3072, n_1=2048, n_2=1024, n_L=10$;
the output layer has activation function Softmax;
and the loss function is crossentropy.
We use Cifar-10 to train the network to test the effect of the $L_{2,\infty}$ normalization
on the over-fitting.
%
We train ten networks:

{\parskip=1pt
$\F_1$: the standard model~\eqref{eq-dnn0}.

$\F_2$: with $L_{1}$ regularization and super parameter 0.0001.

$\F_3$: with $L_{1}$ regularization and super parameter 0.00001.

$\F_4$: with $L_{1}$ regularization and super parameter 0.000001.

$\F_5$: with $L_{2}$ regularization and super parameter 0.001.

$\F_6$: with $L_{2}$ regularization and super parameter 0.0001.

$\F_7$: with $L_{2}$ regularization and super parameter 0.00001.

$\F_8$: with $L_{2,\infty}$ normalization $c=0.4$.

$\F_9$: with $L_{2,\infty}$ normalization $c=0.5$.

$\F_{10}$: with $L_{2,\infty}$ normalization $c=0.6$.
}

The result is in Table \ref{tab-3}.
From the data, we can see that the $L_{2,\infty}$ normalization reduces the over-fitting problem.
Since we use a simple network, the accuracy of each experiment is not high.
%
Also, when $L$ becomes large, the number of the faces of $\F_g$
is very large and we can use smaller $c$.

\begin{table}[H]
\centering
\begin{tabular}{|c|c|c|c|c|c|}
  \hline
  network  & $\F_1$    & $\F_2$    & $\F_3$    & $\F_4$    & $\F_5$    \\
  accuracy & 72.21$\%$ & 52.95$\%$ & 67.16$\%$ & 73.68$\%$ & 61.51$\%$  \\
  \hline
  network   & $\F_6$    & $\F_7$    & $\F_8$    & $\F_9$    & $\F_{10}$ \\
  accuracy & 74.56$\%$ & 72.17$\%$ & 73.36$\%$ & 74.52$\%$ & 73.49$\%$ \\
  \hline
\end{tabular}
\caption{Accuracy of Cifar-10}
\label{tab-3}
\end{table}
%

\section{Conclusion}
\label{sec-conc}
In this paper, we propose to use the $L_{2,\infty}$ normalization to reduce the over-fitting
and to increase the robustness of DNNs.
We give theoretical analysis of the $L_{2,\infty}$ normalization in three aspects:
(1) It is shown that the $L_{2,\infty}$ normalization leads to larger angles
between two adjacent faces of the polyhedron graph of the DNN function
and hence smoother DNN functions, which reduces the over-fitting problem.
(2) A measure of robustness for DNNs is defined and a lower bound
for the robustness measure is given in terms of the $L_{2,\infty}$ norm.
(3) An upper bound for the Rademacher complexity of the DNN
in terms of the $L_{2,\infty}$ norm is given.
Experimental results are  given to show that the
$L_{2,\infty}$ normalization indeed increases the robustness and
the accuracy.

Related with the work in this paper, we propose the following problems
for further study.
The theoretical analysis in this paper is for fully connected DNNs.
For other DNNs, such as the convolution DNN, does the $L_{2,\infty}$ normalization has the same effect?
To obtain smoother DNN functions, we may constrain the derivatives
or the curvature of the DNN function graph. How to impose
such constraints effectively for DNNs?
%
%
Robustness of DNNs was mostly discussed from experimental viewpoint, that is, a network is robust  when it has good accuracy on the validation set with noise. It is desirable to develop a rigours theory and designing framework for robustness DNNs.

%
%

\newpage

\section*{Appendix}

\begin{figure}[ht]
\begin{minipage}[t]{0.49\linewidth}
\centering
\includegraphics[scale=0.5]{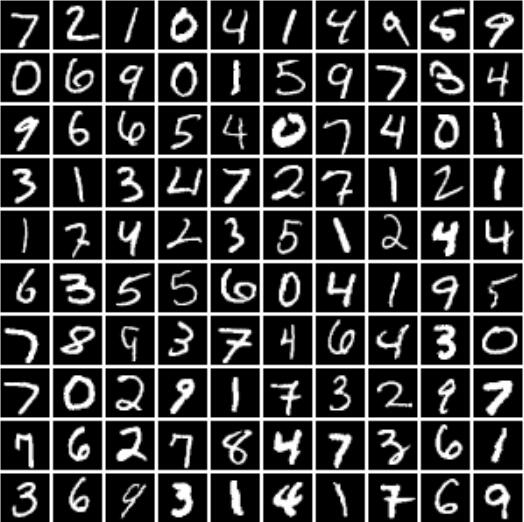}
\caption{No noise. 100 samples}
\end{minipage}
\begin{minipage}[t]{0.48\linewidth}
\hspace{2mm}
\includegraphics[scale=0.5]{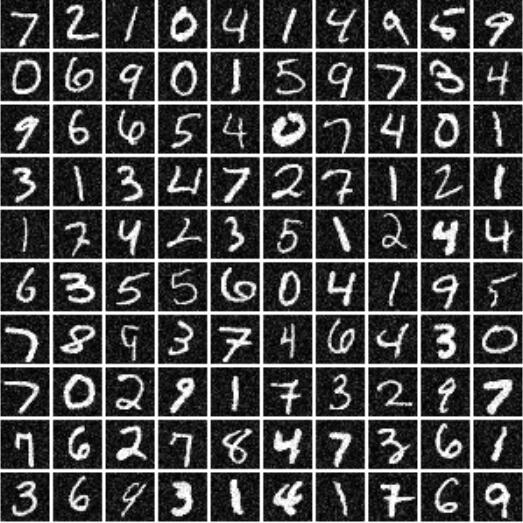}
\caption{First level noise. 100 samples}
\end{minipage}
\end{figure}

\begin{figure}[ht]
\begin{minipage}[t]{0.49\linewidth}
\centering
\includegraphics[scale=0.5]{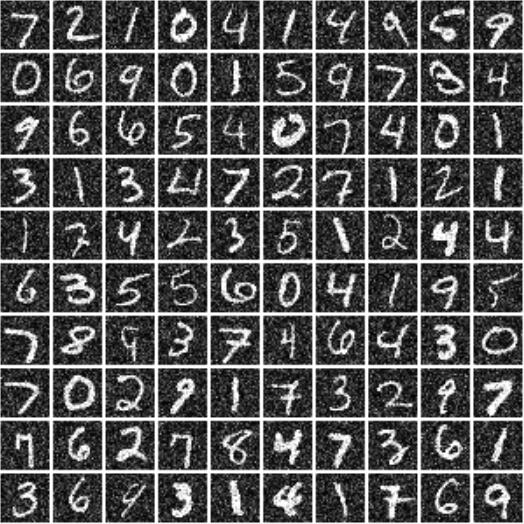}
\caption{Second level noise. 100 samples}
\end{minipage}
\begin{minipage}[t]{0.48\linewidth}
\hspace{2mm}
\includegraphics[scale=0.5]{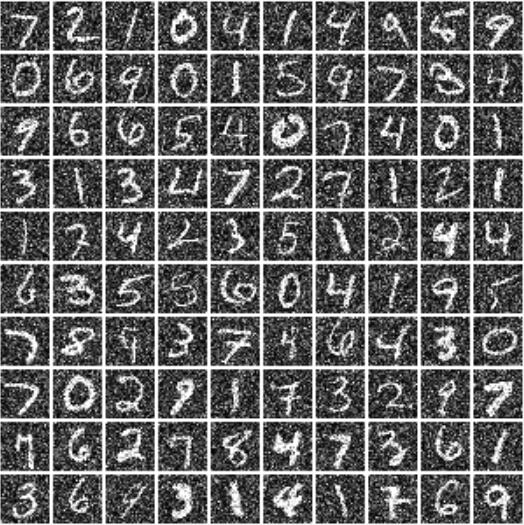}
\caption{Third level noise. 100 samples}
\end{minipage}
\end{figure}

\begin{figure}[ht]
\begin{minipage}[t]{0.49\linewidth}
\centering
\includegraphics[scale=0.5]{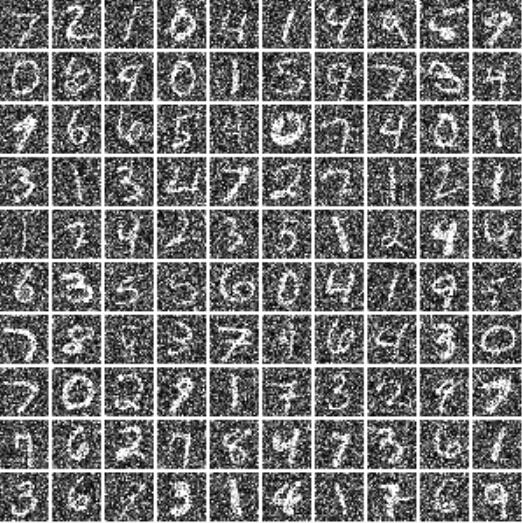}
\caption{Fourth level noise. 100 samples}
\end{minipage}
\begin{minipage}[t]{0.48\linewidth}
\hspace{2mm}
\includegraphics[scale=0.5]{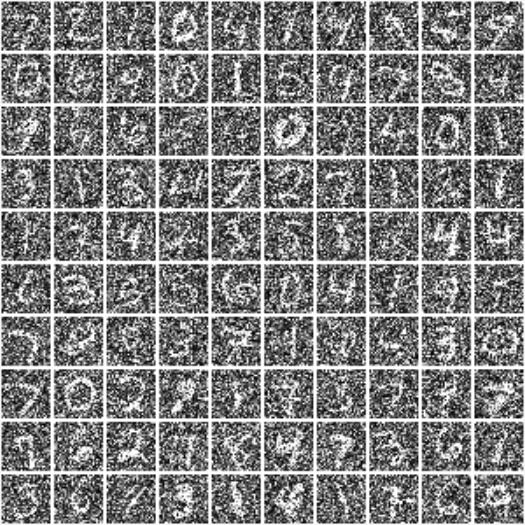}
\caption{Fifth level noise. 100 samples}
\end{minipage}
\end{figure}

\end{document}